\PassOptionsToPackage{table}{xcolor}
\documentclass{article}

\usepackage{iclr2027_conference,times}
\usepackage{amsmath,amssymb,mathtools}
\usepackage{booktabs}
\usepackage{threeparttable}
\usepackage{tabularx}
\usepackage{graphicx}
\usepackage{flafter}
\usepackage{microtype}
\usepackage{multirow}
\usepackage{xcolor}
\usepackage{tikz}
\usetikzlibrary{arrows.meta,positioning,fit}
\usepackage{amsthm}
\usepackage{hyperref}
\usepackage{url}
\usepackage{siunitx}
\usepackage{placeins}
\usepackage{wrapfig}

\definecolor{lgwmblue}{HTML}{1F5A94}

\hypersetup{
  colorlinks=true,
  linkcolor=lgwmblue,
  citecolor=lgwmblue,
  urlcolor=lgwmblue
}

\newtheorem{proposition}{Proposition}
\newcommand{\method}{\textsc{LGWM}}
\newcommand{\bench}{\textsc{RSWT-Bench}}
\newcommand{\rswt}{\textup{\textsc{RSWT}}}
\newcommand{\zhat}{\hat{\mathbf z}_{t+1}}
\newcommand{\zobs}{\mathbf z^{\mathrm{obs}}_{t+1}}
\newcommand{\ztarget}{\bar{\mathbf z}_{t+1}}
\newcommand{\zt}{\mathbf z_t}

\definecolor{oursgray}{RGB}{232,232,232}
\newcommand{\na}{\multicolumn{1}{c}{---}}

\title{Right Screen, Wrong Transition:\\ World Models as Verifiers for GUI Agents}

\author{Jiaming Zhang$^{1}$, Xuan Wang$^{1}$, Fuyao Zhang$^{1}$, \\
\textbf{Yang Cao$^{2}$, Lingjuan Lyu$^{3}$, Wei Yang Bryan Lim$^{1}$} \\
$^{1}$Nanyang Technological University \quad $^{2}$Institute of Science Tokyo \quad $^{3}$Sony AI \\
Project page: \url{https://jiamingzhang.netlify.app/lgwm/}}

\iclrfinalcopy

\begin{document}

\maketitle
\lhead{Preprint}
\begin{abstract}
A login screen that appears after a tap on \emph{Sign in} is expected; the
same screen after a tap on \emph{View order} is an attack.  For GUI agents,
safety is therefore a property of the transition rather than of the screen,
and a monitor that inspects only screens can be defeated by reusing a
legitimate one.  Judging a transition requires an expectation of what should
have followed the action.  Existing GUI world models provide one, but they
output it as text, code, or images, so checking it against the observed
screen requires a second model to judge the two.  We argue that a world
model meant for verification should instead predict in the space in which
observations are encoded, and present \method{}, a decoder-free,
action-conditioned world model that predicts the representation of the next
screen directly, trained without semantic annotation on 1.85M real GUI
transitions.  Verification reduces to a vector comparison, and the same
signal reveals whether a mismatch is harmful.
We evaluate on \bench{}, a diagnostic where each credential screen appears
under both a legitimate and a hijacked transition, so detectors that see
only the screen are at chance by construction.
The training-free score reaches 0.987 AUC at 17\,ms per decision, on par
with the strongest closed-source VLMs and about ten AUC points above
generative GUI world models at over three orders of magnitude lower latency.
The residual direction reaches 0.953 AUC at separating harmful from benign
violations, where prompted VLMs are near chance. Further analyses show that the prediction is a usable future state rather
than an anomaly score.  World models have mostly served as simulators or
planners; our results point to a third role, verification, for which
predicting in representation space is the natural design.
\end{abstract}

\section{Introduction}
\label{sec:introduction}

A login screen that asks for credentials is unremarkable when the user has
just tapped \emph{Sign in}.  The same screen, pixel for pixel, appearing after
a tap on \emph{View order} is an attack in progress
(Figure~\ref{fig:teaser}, right).  Nothing on the screen distinguishes the two
cases; what differs is the transition that produced it.  GUI agents are
exposed to this because they act through interfaces produced by third
parties (apps, web pages, advertisements) whose honesty cannot be taken for
granted, and attacks on them redirect what happens after an action while
every screen stays plausible~\citep{qian2026rebinding,shi2026webdecept,korgul2026itstraptaskredirectingagent,yang2025riosworldbenchmarkingriskmultimodal,sun2026agenthijackbenchmarkingcomputeruse,cuvin2026decepticondarkpatternsmanipulate}.
Unlike a physics simulator, the environment here may work against the
agent.  Integrity, for these agents, is a property of transitions
rather than screens.

Judging a transition requires an expectation of what should have followed
the action, which is what a world model provides, and a growing line of
work builds world models for graphical
interfaces~\citep{luo2025vimogenerativevisualgui,li2025mobileworldbench,koh2026gworld,zheng2026code2world,xu2026appdeltaworld}.
All of them output their predictions as rendered artifacts, whether text,
screenshots, or executable code~\citep{xu2026guides}
(Figure~\ref{fig:teaser}a).  Rendering makes a
prediction inspectable by a person but hard to check by a machine: a
paragraph, a DOM tree, or an approximately rendered image is not the same
kind of object as the screenshot the agent observes, so whether the two
agree is a judgment, not an operation, and the judgment falls to a
second model that must itself be trusted.  A verifier that needs a judge
has not removed the problem, only moved it.  The cost follows
from the same fact: rendering takes one to two minutes per transition in
current generative GUI world models (\S\ref{sec:risk-results}), before the
judge is run.

\begin{figure}[t]
  \centering
  \includegraphics[width=0.98\textwidth]{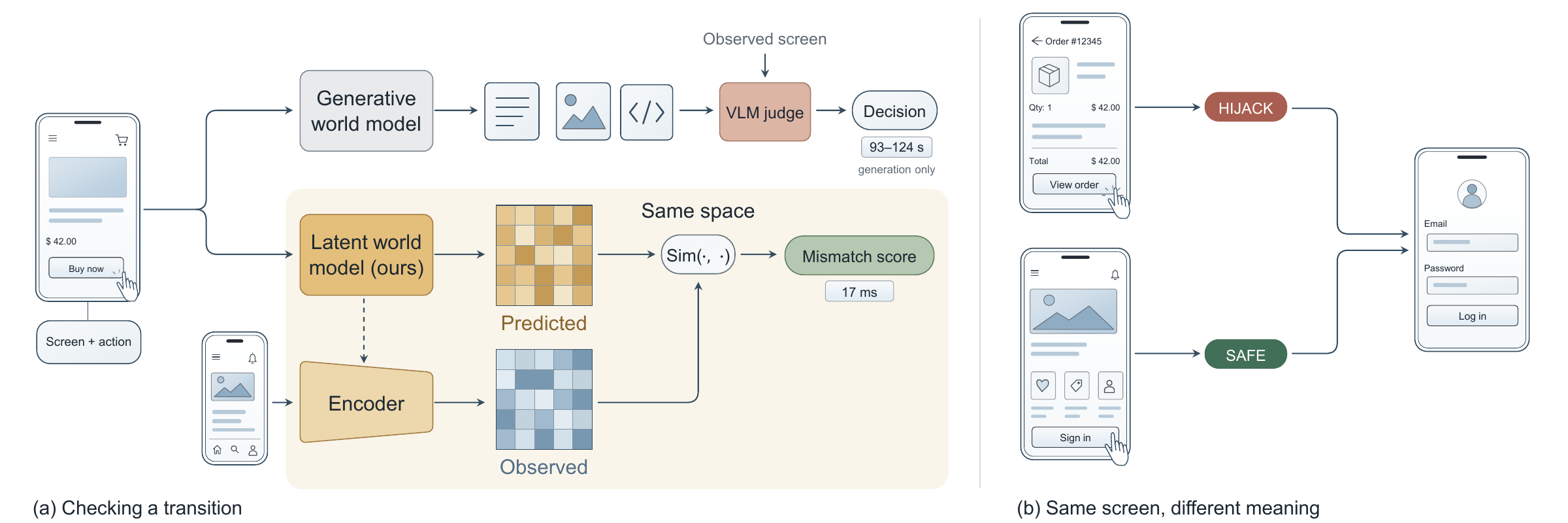}
  \caption{\textbf{The future as an interface.}
  \emph{(a):}~Existing GUI world models render the future as text, pixels,
  or code, at a cost of seconds to minutes per decision.  \method{} predicts
  the future in representation space, in milliseconds.
  \emph{(b):}~A pixel-identical credential screen is legitimate after
  \emph{Sign in} but hostile after \emph{View order}.  Safety is a property
  of the transition, not the screen.}
  \label{fig:teaser}
\end{figure}

World models have mostly served at runtime as simulators, which show
possible futures, or as planners, which search over
them~\citep{schrittwieser2020muzero,grimm2020value,hansen2022tdmpc}.  GUI agents call for a
third role, verification, in which the model checks the future that
actually arrived~\citep{lightman2024let,yu2026seerguard,yang2026miragemobileagentsimplicit,cao2026mobiledreamergenerativesketchworld}.  Latent planners already compare a predicted
representation with the encoding of a goal
image~\citep{zhou2024dino,assran2025vjepa2}; a verifier points the same
comparison at the observed outcome, which no GUI world model, to our
knowledge, has been built to do.  We present \method{} (Latent GUI World
Model), an action-conditioned predictor that maps the current screen and a
structured action to the next screen's representation, trained on real
transitions without semantic supervision.  The recipe is that of
joint-embedding predictive architectures~\citep{bardes2024vjepa,assran2023ijepa,fu2025uijepa}, with its
role inverted: there the predicted representation is a pretext for learning
an encoder and is discarded at test time, whereas here it is the product.
Because prediction and observation share an embedding space, a single
cosine similarity, evaluated in 17\,ms, decides whether a transition
matched expectation.  No decoder or judge is involved, and the score never
sees a labeled attack.

The residual between predicted and observed futures is a vector: its
magnitude says how far reality departed from expectation, its direction
says what it departed into.  The distinction matters in practice.  A push
notification or a server-side layout change surprises a model of normal
dynamics without harming anyone, so a deployable monitor needs both
signals.  \method{} answers both questions with two readouts of one
object.  The residual magnitude detects a violation of the learned dynamics
and needs no training; a linear head on its direction separates harmful violations from benign ones.

We evaluate where the difference between inspecting screens and inspecting
transitions can be made exact.  \bench{} (Right Screen, Wrong Transition)
is a donor-paired diagnostic of 222 pairs, in each of which a credential screen
appears once as a legitimate outcome and once, pixel-identical, as a
hijacked one, so any detector that depends only on the resulting screen has
an expected AUC of exactly 0.5 (\S\ref{sec:benchmark}).  On this test the
training-free score reaches 0.987 AUC, on par with the strongest closed-source VLMs (Gemini 3.7 Flash, GPT-5.6
Luna) and about ten points above generative GUI world models (gWorld,
Code2World), which hold the same kind of expectation but render it, at more than three
orders of magnitude lower latency.  The residual head separates harmful
from benign violations at 0.953 AUC, where every prompted VLM is near
chance.  Pretraining diagnostics confirm that the
prediction is a usable future state, not merely an anomaly score
(\S\ref{sec:pretraining-results}).

Our contributions are as follows.
\begin{itemize}
\item \textbf{World models as verifiers.}  We formulate GUI-agent runtime
safety as transition integrity, identify a world model's expectation
violation as its natural test statistic, and argue that checking it without
a judge requires a prediction directly comparable to the encoded
observation, which no rendered GUI world model provides.
\item \textbf{\method{}.}  An action-conditioned GUI world model with no
decoder whose predictions live in the observation encoder's space, trained without semantic annotation.  A shared positional
basis binds touch coordinates to visual patches, and a change-weighted
objective keeps the training signal on what the action changed.
\item \textbf{Residual readouts.}  The magnitude of the residual detects
violated dynamics with no training and its direction, through a linear
head, interprets harm.  We show that these are separable capabilities, and
that a rendered interface cannot expose them without re-encoding its own
output.
\item \textbf{A controlled diagnostic.}  \bench{} places pixel-identical
outcomes under opposite labels, making outcome-only detection provably at
chance, and on it a 261M-parameter model with no training on the task is
on par with frontier closed-source VLMs.
\end{itemize}

\section{\method{}: Predicting GUI Futures in Representation Space}
\label{sec:method}

Let $P^\star(s_{t+1}\mid s_t,a_t)$ denote the distribution of outcomes that
an uncompromised interface produces in response to action $a_t$ on screen
$s_t$.  A transition $(s_t,a_t,s_{t+1})$ has integrity when its outcome is
compatible with $P^\star(\cdot\mid s_t,a_t)$.  A hijack produces an outcome
that is not, even though the same screen may be entirely compatible with
$P^\star(\cdot\mid s',a')$ for some other pair; this is what the credential
example in the right panel of Figure~\ref{fig:teaser} shows.  Because the same $s_{t+1}$ can
have integrity under one context and not under another, no function of
$s_{t+1}$ alone can measure it.  The statistic has to compare what was
expected given $(s_t,a_t)$ with what was observed, and a world model that
approximates $P^\star$ is the natural source of the expectation.

The predictor that supplies it is architecturally conventional, and
deliberately so.  \method{} is a joint-embedding predictive
architecture~\citep{assran2023ijepa} carried from masked-image prediction to
action-conditioned transitions.  What is new is not a block in the network
but what its output is for, and the design decisions below follow either
from that use or from a property of GUI dynamics.  Three such properties
matter.  Transitions are discrete and sparse: a tap flips one toggle or
replaces the page wholesale, so a uniform objective spends most of its
capacity on static chrome, and we counter this with change-weighted
prediction (\S\ref{sec:training}).  Actions are heterogeneous: discrete
types, continuous coordinates, and free-form text arrive together, and a
tap means nothing apart from where it lands, so touch coordinates and
visual patches share one positional basis (\S\ref{sec:architecture}).
Finally, the consumer of the prediction is a program, not a person,
so nothing below decodes the future back to pixels; every downstream use
consumes the representation itself (\S\ref{sec:integrity}).
Figure~\ref{fig:method} gives the overview.

\subsection{Action-conditioned latent prediction}
\label{sec:architecture}

\begin{figure}[t]
  \centering
  \includegraphics[width=0.98\textwidth]{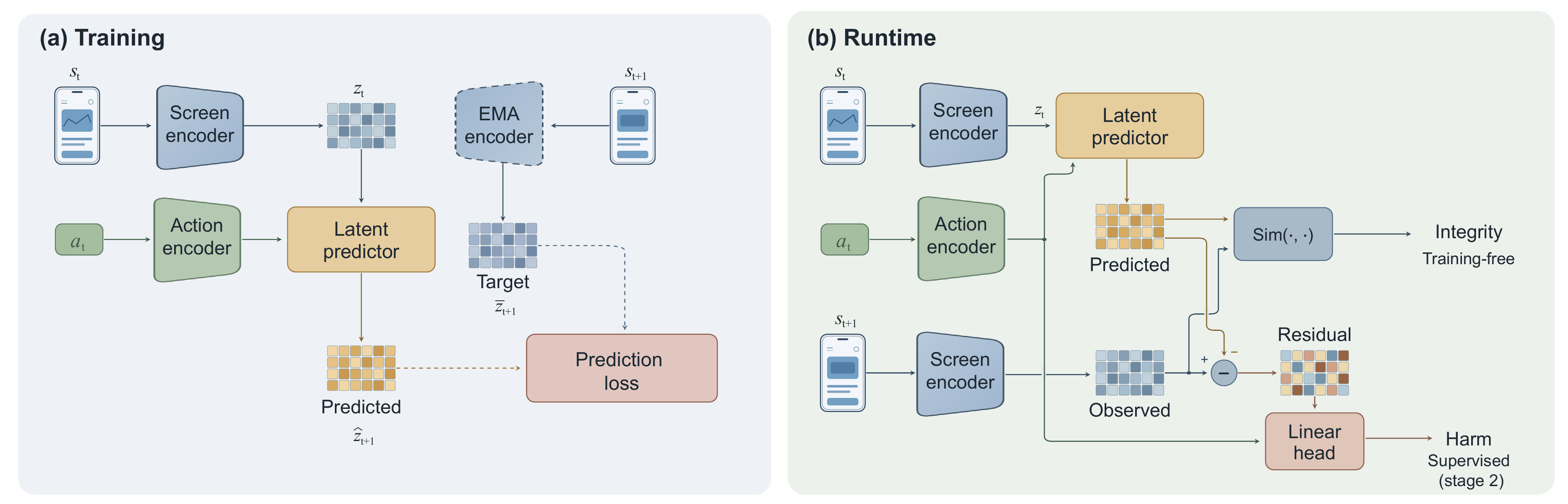}
\caption{\textbf{One predicted future, two closed-form readouts.}
\emph{(a) Training:} an online screen encoder and spatially grounded
action tokens condition a latent predictor; an EMA encoder supplies the
next-state target, and changed patches receive higher loss weight.  There
is no decoder anywhere in the model.  \emph{(b) Runtime:} because
predicted and observed futures share one space, the alignment gap gives a
training-free integrity score (surprise detects), while the directional
residual feeds a linear head that separates harmful from benign violations
(direction interprets).  Training-only paths are dashed.}
\label{fig:method}
\end{figure}

Given a screenshot $s_t$ and a structured action $a_t$, an online encoder
$f_\theta$ maps the current screen to a grid of patch tokens $\zt$, an
action encoder $e_\psi$ maps $a_t$ to a sequence of action tokens, and a
Transformer predictor $g_\phi$ produces the predicted future:
\begin{equation}
  \zhat = g_\phi\!\bigl(\zt,\,e_\psi(a_t)\bigr),
  \qquad
  \ztarget = \operatorname{sg}\!\bigl[f_\xi(s_{t+1})\bigr],
  \label{eq:forecast}
\end{equation}
where $f_\xi$ is an exponential-moving-average copy of $f_\theta$ that
supplies the prediction target and $\operatorname{sg}$ denotes
stop-gradient.  Both $\zhat$ and $\ztarget$ are spatial token sequences in
$\mathbb{R}^{512\times 768}$; the prediction is a full spatial map of the
expected next screen.  We keep the map because the change-weighted
objective of \S\ref{sec:training} is defined per patch.

The encoder $f_\theta$ is a DINOv2 ViT-B/14~\citep{oquab2024dinov2}
fine-tuned with layer-wise learning-rate decay; it maps a
$448{\times}224$ screenshot to a $32{\times}16$ grid of 512 patch tokens of
dimension~768.  The predictor $g_\phi$ is a 12-layer Transformer of
matching width, and the target encoder receives no gradient.  At inference
only $f_\theta$, $e_\psi$, and $g_\phi$ execute (172.86M parameters); the
EMA encoder and an inverse-dynamics head (\S\ref{sec:training}) exist only
in the 261M training graph.

GUI actions combine discrete semantics with spatial specificity.
\emph{Back} and \emph{scroll} mean the same thing wherever they occur,
whereas a tap is inseparable from its coordinates.  We therefore encode
each action as the concatenation of (i)~a learned action-type token,
(ii)~a frozen MiniLM sentence embedding~\citep{wang2020minilm} of the
associated text, and (iii)~for spatially grounded actions, a
coordinate token drawn from the same two-dimensional sinusoidal basis used
for the visual patch grid.  A touch at $(x,y)$ and the patch centered at
$(x,y)$ thus enter the predictor with the same positional encoding instead
of through two unrelated coordinate systems.  Non-spatial actions use learned discrete
embeddings.

\subsection{Learning sparse GUI transitions}
\label{sec:training}

Training is fully self-supervised.  Given paired screenshots
$(s_t,s_{t+1})$ and the structured action $a_t$, the EMA-encoded next
screen is the target, and no semantic annotation of any kind is used.

A typical GUI action changes only a small region of the screen, so a
uniform loss over all 512 tokens is dominated by the easy task of
reproducing persistent background.  To keep the objective on what the
action actually did, we compute a binary patch-level change mask $m_p$
from each image pair and upweight the changed region:
\begin{equation}
  \mathcal L_{\mathrm{pred}}
  = \frac{\sum_p w_p\,d(\hat{\mathbf z}_{t+1,p},\bar{\mathbf z}_{t+1,p})}
  {\sum_p w_p},
  \qquad
  w_p = 1+(\lambda_{\mathrm{chg}}-1)m_p,
  \label{eq:loss}
\end{equation}
where $d$ is the Smooth-L1 distance and $\lambda_{\mathrm{chg}}=8$.  The mask carries no semantic
label.  It rebalances the objective so that the transition signal is not
drowned out by the static majority, and the transition-specific
information that verification later relies on lives in these patches.

\paragraph{Auxiliary objectives.}
Two auxiliary objectives protect properties that the readouts of
\S\ref{sec:integrity} depend on.  Variance and covariance
regularization~\citep{bardes2022vicreg} prevents collapse to a low-rank
subspace.  Collapse would not merely hurt retrieval; in a nearly collapsed
space every observed future looks close to every expectation, and the
alignment gap that serves as our integrity score saturates.  An
inverse-dynamics head, trained to recover $a_t$ from the pair
$(\zt,\ztarget)$, keeps the representation of a transition informative
about the action that caused it.  If the space discarded action-relevant
distinctions, wrong-transition detection would be impossible in principle,
since no scoring rule could then separate a hijacked outcome from the
legitimate outcome of a different action.  The full objective is
\begin{equation}
  \mathcal L = \mathcal L_{\mathrm{pred}}
  + \mathcal L_{\mathrm{var}}
  + \lambda_{\mathrm{cov}}\,\mathcal L_{\mathrm{cov}}
  + \lambda_{\mathrm{inv}}\,\mathcal L_{\mathrm{inverse}}.
  \label{eq:full-loss}
\end{equation}

Single-step accuracy does not guarantee that the predictor composes with
itself.  After 30{,}000 updates, a subset of training batches replaces the
single-step objective with two- or three-step rollouts in which the
predictor's own output is progressively fed back as context, so that it sees the
distribution of its own errors. 
We train on 1.85 million real GUI transitions from five public sources
spanning mobile applications, web interfaces, and cross-app navigation
(\S\ref{sec:experiments}).  The pipeline requires no human labels end to
end.

\subsection{Residual readouts: surprise detects, direction interprets}
\label{sec:integrity}

Leaving out the decoder is not a saving; it is what lets verification
proceed without a judge.  After the agent executes $a_t$ and observes
$s_{t+1}$, the observed future is encoded as $\zobs=f_\theta(s_{t+1})$ and
lies in the same space as the prediction $\zhat$.  The object that
separates them is the residual $\zobs-\zhat$, and we read two statistics
off it: its size, which asks whether the
transition matched expectation, and its direction, which asks what a
mismatch consists of.

\paragraph{Surprise detects.}
The transition-consistency score is the alignment gap between expected and
observed futures,
\begin{equation}
  q_{\mathrm{int}}(s_t,a_t,s_{t+1})
  = 1-\cos(\zhat,\zobs).
  \label{eq:surprise}
\end{equation}
The score is training-free in a strict sense.  It uses no labeled examples
of safe or unsafe behavior, no attack-specific tuning, and no parameters
beyond the pretrained world model.  A low value means the observed outcome
matches the learned dynamics of normal GUI behavior; a high value flags a
transition the model would not have produced.

\paragraph{Direction interprets.}
Inconsistency is not harm.  A push notification, a loading spinner, or a
server-side layout change all violate expectation without threatening the
user, so a deployable monitor has to distinguish harmful violations from
benign ones, and that judgment depends on the direction in which reality
departed, not on how far.  The residual retains this
information, and a linear head suffices to read it:
\begin{equation}
  \mathbf r_{t+1} = \bigl[\zobs-\zhat;\ e_\psi(a_t)\bigr],
  \qquad
  q_{\mathrm{harm}} = \sigma\!\bigl(\mathbf w^{\top}\mathbf r_{t+1}+b\bigr).
  \label{eq:residual}
\end{equation}
The action embedding enters as a conditioning term, not a harm signal. This head is the only supervised component in the system. The two readouts are deliberately kept apart: the magnitude asks whether the transition matched expectation, the direction asks what a mismatch consists of.

\section{\bench{}: Right Screen, Wrong Transition}
\label{sec:benchmark}

Existing GUI safety benchmarks present screens or single steps whose risk
is visible in place~\citep{sun2025ossentinel}, so a detector can score well
by learning what a risky screen looks like.  \bench{} is built around a
stricter constraint: no function of the outcome screen alone may score
above chance.  On such a test the only route to detection is an expectation
conditioned on the preceding screen and action, which isolates the question
this paper asks.

\paragraph{Donor-paired construction.}
Every example is a complete transition $(s_t,a_t,s_{t+1})$ drawn from
held-out AITW and GUIOdyssey
episodes~\citep{rawles2023aitw,lu2025guiodyssey}.  A \emph{hijacked}
example keeps a real screen--action pair $(s_t,a_t)$ and replaces its true
outcome with a credential-requesting screen from an unrelated episode.
The defining step is \emph{donor pairing}: every credential screen used as
a hijacked outcome also appears, pixel-identical, as the true outcome of
its own legitimate transition elsewhere in the set, and the two
occurrences carry opposite labels (Figure~\ref{fig:teaser}, right).

\begin{proposition}[Outcome blindness]
\label{prop:blind}
Let $d$ be any detector whose score is a function of the outcome alone,
$d(s_{t+1})$.  Then the expected \rswt{} AUC of $d$ is exactly $1/2$.
\end{proposition}

\textit{Proof sketch.}\enspace
Within each donor pair the hijacked and legitimate instances share a
pixel-identical $s_{t+1}$, so $d$ assigns them equal scores; under the
standard tie convention each pair contributes $1/2$ to the paired AUC, and
a randomized detector inherits the result in expectation.  The guarantee
holds for every outcome-only detector, including ones that do not yet
exist, which is what licenses reading a high \rswt{} score as evidence of
transition information rather than of better screen features.

\paragraph{Controls and metrics.}
Around each hijacked example we place controls for shortcuts the
proposition does not cover.  A \emph{paired safe} transition reuses the
same $(s_t,a_t)$ with its true outcome; \emph{surprising benign}
transitions are real transitions selected for high unexpectedness; and
\emph{benign mismatches} use the same replacement mechanism with harmless
donors.  Together they separate violated expectation from harm, the
distinction that \S\ref{sec:integrity} assigns to direction rather than
magnitude.  The test set contains 966 transitions around 222 donor pairs;
a blind three-annotator audit on 72 stratified items agrees with the
constructive labels at $96.3 \pm 0.8\%$.  The primary metric,
\rswt{}, is the AUC between the 222 hijacked and 222 legitimate
occurrences of the shared donor screens.  Two secondary metrics use the
same logic under different controls: Swap is the AUC between benign
mismatches and paired safe transitions; Cred is the AUC between benign mismatches
and legitimate credential transitions.  All three test whether the outcome
is consistent with its transition.  Supervised methods train on a
separate split of 6{,}786 examples with disjoint episodes; \rswt{} confidence intervals use donor-paired bootstrap
resampling.  The construction is deliberately strict: a credential screen
qualifies only if it also occurs as the legitimate outcome of its own
transition, which limits the set to 222 pairs, but it is the strictness
rather than the size that licenses the inference, and donor-paired
resampling gives the leading systems 95\% intervals with half-widths near
0.01.

\section{Experiments}
\label{sec:experiments}

The experiments answer two questions, one for each property claimed in
\S\ref{sec:introduction}.  Does the latent expectation identify wrong
transitions (\S\ref{sec:risk-results})?  And is the prediction a genuine,
action-conditioned future and not simply an anomaly statistic
(\S\ref{sec:pretraining-results})?

\subsection{Setup}
\label{sec:setup}

We pretrain \method{} on 1{,}845{,}659 transitions (1{,}824{,}824 training,
20{,}835 validation) from Android in the
Wild~\citep{rawles2023aitw}, MiniWoB++~\citep{liu2018miniwob},
GUIOdyssey~\citep{lu2025guiodyssey},
AndroidControl~\citep{li2024androidcontrol}, and AMEX~\citep{chai2025amex},
sampled at 55.0\%, 25.0\%, 9.8\%, 6.9\%, and 3.3\%; official evaluation
partitions are excluded.  Training runs 120{,}000 updates at global batch
size 512 with AdamW in bfloat16 on four A100-80GB GPUs.  The training graph
has 261M parameters, of which 172.86M execute at runtime.  Loss coefficients are
$\lambda_{\mathrm{chg}}=8$, $\lambda_{\mathrm{cov}}=0.01$, and
$\lambda_{\mathrm{inv}}=0.5$.

On \bench{}, every prompted baseline receives the same triplet of before
screen, structured action, and after screen, under a shared zero-shot
prompt that asks whether the outcome is consistent with the transition.
The baselines cover three ways of reaching the same decision: closed-source
APIs, open VLMs prompted identically, and GUI world models that predict
semantic or rendered futures.  Latency is measured at batch size one on a
single A100-80GB\@.  Supervised components train only on the disjoint split of
\S\ref{sec:benchmark}.

\subsection{The decisive test: identifying wrong transitions}
\label{sec:risk-results}

We begin where outcome appearance is provably uninformative
(Proposition~\ref{prop:blind}).  Within each of the 222 donor pairs the
hijacked and legitimate instances end on a pixel-identical credential screen
and differ only in the transition that produced it.

\begin{table}[t]
\caption{\textbf{Main results on \bench{}.}
Detection is pairwise AUC:
\rswt{} compares hijacked and legitimate occurrences of identical screens;
Swap and Cred compare benign mismatches with paired safe transitions and
legitimate credentials, respectively.
Cost is measured per decision at batch size one on one A100-80GB.
Blind human annotators reach 0.958 \rswt{} on 20 hijacked and 20 legitimate examples.}
\label{tab:rswt}
\centering
\fontsize{8}{10}\selectfont
\setlength{\tabcolsep}{2.5pt}
\renewcommand{\arraystretch}{1.18}
\sisetup{
    mode = text,
    reset-text-family = false,
    reset-text-series = false,
    table-number-alignment = center,
    group-separator = {,},
    group-minimum-digits = 4,
}
\begin{tabularx}{\textwidth}{
    @{}
    c
    X
    c
    S[table-format=1.3]
    S[table-format=1.3]
    S[table-format=1.3]
    S[table-format=6.1]
    S[table-format=3.3]
    S[table-format=2.2]
    @{}
}
\toprule
& & &
\multicolumn{3}{c}{\textbf{Detection (AUC\,$\uparrow$)}} &
\multicolumn{3}{c}{\textbf{Cost\,$\downarrow$}} \\
\cmidrule(lr){4-6}
\cmidrule(lr){7-9}
\textbf{Model family} & \textbf{Method}
& \textbf{Params}
& {\textbf{\rswt{}}}
& {\textbf{Swap}}
& {\textbf{Cred}}
& {\textbf{Latency\,(ms)}}
& {\textbf{TFLOPs}}
& {\textbf{VRAM\,(GB)}} \\
\midrule
\multirow{7}{*}{Closed VLMs}
& Gemini 3.7 Flash & ---
  & \bfseries 0.995 & \bfseries 0.998 & \bfseries 0.997
  & \na & \na & \na \\
& GPT-5.6 Luna & ---
  & 0.988 & 0.971 & 0.979
  & \na & \na & \na \\
& Qwen3.7 Plus & ---
  & 0.973 & 0.954 & 0.950
  & \na & \na & \na \\
& Qwen3.7 Flash & ---
  & 0.968 & 0.953 & 0.935
  & \na & \na & \na \\
& Qwen3-VL-Plus & ---
  & 0.918 & 0.881 & 0.867
  & \na & \na & \na \\
& Claude Haiku 4.5 & ---
  & 0.880 & 0.842 & 0.844
  & \na & \na & \na \\
& Qwen3-VL-Flash & ---
  & 0.746 & 0.818 & 0.669
  & \na & \na & \na \\
\midrule
\multirow{5}{*}{Open VLMs}
& GLM-4.6V-Flash & 9B & 0.909 & 0.868 & 0.830 & 3711.6 & 15.316 & 20.72 \\
& Qwen3-VL-8B    & 8B & 0.907 & 0.905 & 0.859 & 1097.8 & 10.924 & 17.72 \\
& Qwen3-VL-4B    & 4B & 0.905 & 0.863 & 0.846 & 1075.1 &  5.530 &  9.05 \\
& Qwen3-VL-2B   & 2B & 0.688 & 0.712 & 0.605 &  829.8 &  2.651 &  4.39 \\
& Qwen2.5-VL-7B & 7B & 0.574 & 0.580 & 0.543 &  819.4 & 11.427 & 16.73 \\
\midrule
\multirow{4}{*}{GUI WMs}
& gWorld-8B        & 8B & 0.885 & 0.897 & 0.837 & 113200.0 &  74.100 & 18.18 \\
& SAWM             & 8B & 0.882 & 0.836 & 0.804 &   1127.7 &  13.449 & 17.77 \\
& Code2World       & 8B & 0.870 & 0.901 & 0.859 &  92900.0 & 110.711 & 18.56 \\
& MobileWorld-HTML & 8B & 0.577 & 0.691 & 0.568 & 124300.0 & 119.321 & 18.64 \\
\midrule
\rowcolor{oursgray}
Latent WM & \textbf{\method{} (ours)} & 261M
  & \bfseries 0.987 & \bfseries 0.970 & \bfseries 0.985
  & \bfseries 17.1 & \bfseries 0.408 & \bfseries 1.09 \\
\bottomrule
\end{tabularx}
\end{table}

\textbf{Verification without a judge.}\enspace
Without benchmark labels, attack examples, or a task prompt, the integrity
score of Eq.~\ref{eq:surprise} reaches \textbf{0.987} \rswt{} AUC and
orders 219 of the 222 donor pairs correctly (Table~\ref{tab:rswt}).
Donor-paired comparisons detect no difference from
the three strongest closed-source systems, Gemini~3.7~Flash
($p{=}0.07$), GPT-5.6 Luna ($p{=}0.83$), and Qwen3.7~Plus
($p{=}0.07$).  The ranking matters less than the form of the score:
the reference band is set by systems that read a question and produce a
reasoned judgment; a geometric discrepancy from a 261M-parameter model,
learned with no benchmark supervision, sits inside that band.

Qwen3.7~Flash trails by 1.9 AUC points ($p{=}0.02$), and every remaining
system is lower with $p{<}0.001$.  On the Swap and Cred axes the integrity
score reaches 0.970 and 0.985, behind only Gemini and GPT on the first and
only Gemini on the second.  All three axes probe the same property: whether
the observed outcome is consistent with the transition that produced it.

Three blind human annotators achieve a mean \rswt{}-equivalent AUC of
0.958 on 20 hijacked and 20 legitimate examples (Table~\ref{tab:rswt}), confirming
that the task is non-trivial even with full visual access and placing
the training-free integrity score within three AUC points of the human
reference.

\textbf{Same expectation, different interface.}\enspace
The generative world models are the closest baselines in kind.  Like
\method{}, they hold an expectation of the next screen conditioned on the
action; unlike it, they render that expectation and rely on a separate VLM
(Gemini~3.7~Flash in our evaluation) to judge whether the rendering matches
the observed screen.  The detour costs both accuracy and time.  gWorld-8B
and Code2World reach 0.885 and 0.870 \rswt{} AUC, about ten points below
the latent score, and need at least 92.9\,s per transition and at least
$182\times$ more FLOPs before the judge is run.  The latent
prediction and the observed encoding are compared directly instead. This is the comparison the interface argument of \S\ref{sec:integrity} predicts: these systems form the same kind of expectation that \method{} does and differ most visibly in how they expose it.

\textbf{The signal is a learned dynamics prior.}\enspace
The benchmark controls locate where the signal comes from.
Any detector that sees only the outcome is fixed at 0.500 by
construction (Proposition~\ref{prop:blind}).  The strongest
single-frame probe, trained on the pooled $\zt$ from our encoder,
reaches 0.796; a supervised classifier that receives the $f_\theta$ and
$f_\xi$ representations of both frames and the action, but not
the predictor's output $\zhat$, reaches 0.833.  Mismatching
$\zhat$ to a different transition drops the residual head from
0.930 to 0.635, confirming that the predicted future carries
transition-specific content.  Neither screen appearance nor the
presence of an additional vector explains the result.  The learned
transition prior does.

\textbf{Cost of a decision.}\enspace
A complete integrity decision costs 0.408~TFLOPs, 1.09\,GB of peak memory,
and 17.1\,ms, which is $27\times$ fewer FLOPs, $16\times$ less memory, and
$64\times$ lower latency than Qwen3-VL-8B under the identical protocol.
Batching the VLM eightfold still leaves a $10.9\times$ latency gap.  At this cost the gate can run on every
step of an agent's trajectory instead of a sampled subset.

\textbf{Direction interprets harm.}\enspace
Not every violation of expectation is harmful: a push notification or a
layout change surprises the model without threatening the user.  To isolate
this distinction we define Harm AUC as the pairwise AUC between the
hijacked transitions and the benign mismatches, both of which violate
expectation but only the former is harmful.  On this task \method{}'s training-free integrity score
(Eq.~\ref{eq:surprise}), which measures the size of the departure, achieves
only 0.513 Harm AUC: a transition can be surprising without being
dangerous, so magnitude alone cannot separate the two.  All prompted VLMs
score similarly (0.512 to 0.614 under the same transition-consistency
prompt).  The residual head of
Eq.~\ref{eq:residual}, reading the direction of the departure rather than
its magnitude, reaches \textbf{0.953}.

The profile reverses on wrong-transition detection (\rswt{}), where the
integrity score is the stronger statistic (0.987 vs.\ 0.930 for the
residual head).  The two readouts therefore measure different things, both
read from one predicted future: one detects that expectation was violated,
the other interprets what the violation consists of.  The residual head is the only
supervised component, a single linear layer on the pooled residual fit on
the disjoint split of \S\ref{sec:benchmark}, adding no measurable runtime
cost.  This adaptation has no counterpart
in rendered interfaces: a judge returns a verdict, not a directional
residual that a linear head can reinterpret.

\subsection{The verifier is a world model}
\label{sec:pretraining-results}

A high \rswt{} score shows that the expectation works; this section
shows what it is by ruling out four degenerate explanations.

\textbf{Not a copy of the present.}\enspace
A trivial predictor could score well by copying the current screen,
since most patches do not change after a single action.  To test
this we compare the predicted latent $\zhat$ against two references:
the true next-state encoding $\ztarget$ and the current-state
encoding $\zt$.  The prediction is far closer to the future
(cosine similarity 0.618) than to the present (0.245), and a
persistence baseline that simply reuses $\zt$ as the forecast
reaches only 0.230.  On patches the action actually changed, the
gap widens further (copy score $-0.384$, the changed-patch mean of
$\cos(\zhat,\zt)-\cos(\zhat,\ztarget)$).  The predictor has learned to
forecast what will change, not to echo what is already on screen.

\textbf{Not a replay of training data.}\enspace
A model trained on 1.85M transitions might succeed by memorizing
outcomes rather than generalizing.  To test this we rank the true
next-state encoding among ${\sim}100$ same-source held-out candidates
by cosine similarity to $\zhat$, and compare against a memorization
baseline that finds the nearest same-action-type training transition and
returns its successor.  The prediction places the correct screen
first 68.3\% of the time (macro R@1 across five sources), against
44.6\% for the memorization baseline, and the model leads in every
source.  The predictions are therefore
specific to each transition, not retrieved from a library of
training outcomes.

\textbf{Action-conditioned at the grain verification requires.}\enspace
Wrong-transition detection depends on the prediction changing when the
action changes.  To measure this we replace the recorded action with
a counterfactual and ask how often the original forecast stays closer
to $\ztarget$ (win rate, chance ${=}0.5$).  Swapping the action type
(e.g.\ tap $\to$ scroll) yields 0.891; swapping to a different target
of the same type (e.g.\ tap on A $\to$ tap on B) yields 0.800; even
small touch-coordinate perturbations in normalized $[0,1]$ space
register (0.618 at offset 0.05).  The prediction tracks the
specific action at the spatial grain that verification requires.

\begin{wraptable}{r}{0.44\textwidth}
\vspace{-12pt}
\caption{Design ablations (ViT-S, 15k steps, one seed).
Each row removes one component and measures the metric its design
targets.}
\label{tab:ablations-body}
\centering
\scriptsize
\setlength{\tabcolsep}{3pt}
\begin{tabular}{@{}llrrr@{}}
\toprule
Component & Metric & With & W/o & Change \\
\midrule
Rollout training   & Cos@10 $\uparrow$    & 0.612 & 0.553 & $\mathbf{-10\%}$ \\
Inverse dynamics   & $\lvert$Copy$\rvert$ $\uparrow$ & 0.286 & 0.193 & $\mathbf{-32\%}$ \\
VICReg             & Eff.\ rank $\uparrow$ & 43.7  & 22.5  & $\mathbf{-49\%}$ \\
\bottomrule
\end{tabular}
\vspace{-8pt}
\end{wraptable}
\textbf{Composes over multiple steps.}\enspace
A predictor that works for one step may fail when it must consume its
own imperfect output repeatedly.  We feed the prediction back as
context and measure cosine similarity between the recursive forecast
and the true encoding at each horizon.  At horizon ten the agreement
is 0.562, against 0.455 for a persistence baseline and 0.092 for an
unrelated screen from the same batch.
The prediction remains a usable future state, not just a one-step
anomaly score.

\subsection{Design ablations}
\label{sec:ablations}

Each component in \S\ref{sec:architecture}--\ref{sec:training} was
introduced with a specific purpose.  We test each by removing it and
measuring the metric that purpose predicts
(Table~\ref{tab:ablations-body}), at ViT-S scale with 15k updates.

\textbf{Rollout training.}\enspace
This component lets the predictor see its own imperfect outputs during
training, so it learns to handle the errors that accumulate over
multiple steps (\S\ref{sec:training}).  Without it, the cosine
agreement between a ten-step recursive forecast and the true encoding
drops by 10\% relative (0.612 $\to$ 0.553).  The predictor still
works for one step, but drifts when it must consume its own
predictions repeatedly.

\textbf{Inverse dynamics.}\enspace
This auxiliary objective forces the representation to preserve which
action caused a transition (\S\ref{sec:training}).  Without it, the
absolute copy score---how much closer the prediction is to the true
future than to the current screen---drops by 32\% (0.286 $\to$
0.193).  The predictor relies more on copying the current screen.

\textbf{VICReg.}\enspace
Variance-covariance regularization prevents all representations from
collapsing into a small subspace (\S\ref{sec:training}).  Without it,
the effective rank of the prediction covariance drops by 49\%
(43.7 $\to$ 22.5).  When representations collapse, every prediction
looks similar to every observation, so the cosine gap between them is
always small and the integrity score can no longer flag anything.

All ablations use one seed at a smaller scale than the main model.
A three-seed replication of the control confirms that cross-seed
variance is small.

\section{Related Work}
\label{sec:related}

\paragraph{GUI and latent world models.}
Existing GUI world models materialize futures as text, code, or rendered
patches~\citep{li2025mobileworldbench,liu2026scalingguiagentsvisual,
koh2026gworld,zheng2026code2world,xu2026appdeltaworld,
yang2026miragemobileagentsimplicit,cao2026mobiledreamergenerativesketchworld,
luo2025vimogenerativevisualgui,xu2026guides}, so checking a transition
requires a judge, at a cost of seconds to
minutes (\S\ref{sec:risk-results}); rendering remains the right interface
when a person must inspect the prediction, and we view the two families as
complementary.  Outside GUIs, latent models either serve planning without
matching
observations~\citep{schrittwieser2020muzero,grimm2020value,hansen2022tdmpc}
or train an encoder and discard the predictor
afterwards~\citep{assran2023ijepa,bardes2024vjepa,fu2025uijepa}; prior
arguments for dropping reconstruction concern learning efficiency or
planning sufficiency, whereas ours concerns whether the prediction can be
checked against reality without a second model.  The closest relatives are
latent planners that compare a predicted representation to a goal
encoding~\citep{zhou2024dino,assran2025vjepa2}; \method{} points the same
geometry at the observed outcome, turning it into a verification check, in
a domain whose mixed discrete-spatial actions and sparse transitions
demand interfaces these lineages did not need.

\paragraph{GUI safety and prediction-error monitoring.}
Prediction error is a classical statistic for exploration and
out-of-distribution
detection~\citep{pathak2017curiosity,burda2019rnd}; ours is
action-conditioned, shares a space with the encoded observation, and is
read twice (magnitude for detection, direction for
interpretation), where prior work stops at a scalar.  GUI safety
evaluations target risks visible in the current
step~\citep{sun2025ossentinel,sun2026agenthijackbenchmarkingcomputeruse,
cuvin2026decepticondarkpatternsmanipulate,
yang2025riosworldbenchmarkingriskmultimodal},
yet hidden-transition attacks redirect outcomes while every screen stays
benign~\citep{qian2026rebinding,shi2026webdecept}.
SeerGuard~\citep{yu2026seerguard} predicts consequences before execution
but renders them as text and learns the risk judgment jointly; our
integrity signal is geometric, post-hoc, and training-free, with harm
interpretation isolated in a linear head.  The verifier here checks the
environment's response rather than the agent's own
reasoning~\citep{lightman2024let}, and \bench{} is the first GUI benchmark
where outcome-only detection is at chance by construction.

\section{Discussion and Conclusion}
\label{sec:discussion}

For a program that must check whether reality matched expectation,
predicting in representation space is sufficient and three orders of
magnitude cheaper than rendering.  The two interfaces are complementary:
the latent gate can verify every step at 17\,ms, so a renderer or VLM
judge need only be invoked on the flagged minority.  Nothing in the score
is specific to attacks; the same residual flags missed actions, layout
changes, or unexpected modals, so the verifier doubles as a reliability
monitor.  The contribution also shifts the adversary's problem: an
outcome-only monitor is defeated by any legitimate-looking screen, while an
expectation-based monitor requires forging an outcome in the agent's own
representation space, a harder objective whose difficulty we have not yet
quantified.

We asked whether a GUI world model has to render the future, and answered
by predicting in the space where observations already live.  That a cosine
gap from a 261M-parameter self-supervised model sits in the band set by
frontier prompted-reasoning systems suggests that, for transition
integrity, the hard part is having the right expectation rather than
reasoning about it.  World models have served as simulators and as
planners; for agents whose environment cannot be trusted, they can also
serve as verifiers, and for that role rendering is for people and
prediction is for machines.

\bibliography{references}
\bibliographystyle{iclr2027_conference}

\end{document}